\documentclass{article}
\usepackage{amsmath,amssymb,amsthm}
\usepackage{multirow}
\usepackage{graphicx} 
\usepackage{subfigure} 

\usepackage{natbib}

\usepackage{algorithm}
\usepackage{algorithmic}

\usepackage[dvipdfm]{hyperref}

 \newcommand{\RR}{\mathbb{R}}

 \newcommand{\T}{{}^\top}

 \newcommand{\mC}{\boldsymbol{C}}
 \newcommand{\mE}{\boldsymbol{E}}
 
 \newcommand{\mS}{\boldsymbol{S}}
 
 \newcommand{\mW}{\boldsymbol{W}}
 \newcommand{\mX}{\boldsymbol{X}}
 \newcommand{\mY}{\boldsymbol{Y}}
 \newcommand{\mZ}{\boldsymbol{Z}}

 \newcommand{\mU}{\boldsymbol{U}}
 \newcommand{\mV}{\boldsymbol{V}}
 \newcommand{\mI}{\boldsymbol{I}}

 \newcommand{\mDelta}{\boldsymbol{\Delta}}

 \newcommand{\tA}{\mathcal{A}}
 \newcommand{\tB}{\mathcal{B}}
 \newcommand{\tC}{\mathcal{C}}
 \newcommand{\tD}{\mathcal{D}}
 \newcommand{\tE}{\mathcal{E}}
 \newcommand{\tG}{\mathcal{G}}
 \newcommand{\tX}{\mathcal{X}}
 \newcommand{\tY}{\mathcal{Y}}
 \newcommand{\tW}{\mathcal{W}}

 \newcommand{\tDelta}{\Delta}

\newcommand{\norm}[1]{\bigl|\!\bigl|\!\bigl|#1\bigr|\!\bigr|\!\bigr|}

\newcommand{\tr}{\underline{r}}
\newcommand{\lr}{\overline{r}}

 \def\dot#1#2{\left\langle #1,#2\right\rangle}

 \newcommand{\minimize}{\mathop{\rm minimize}}
 
 \newcommand{\argmin}{\mathop{\rm argmin}}

 \newcommand{\Secref}[1]{Section~{\ref{#1}}}

\newtheorem{lemma}{Lemma}
\newtheorem{theorem}{Theorem}

\newtheorem{corollary}{Corollary}

\usepackage{icml2013} 

\icmltitlerunning{Tensor Decomposition via Structured Schatten Norm Regularization}

\begin{document} 

\twocolumn[
\icmltitle{
Convex Tensor Decomposition via
Structured Schatten Norm Regularization
}

\icmlauthor{Your Name}{email@yourdomain.edu}
\icmladdress{Your Fantastic Institute,
            314159 Pi St., Palo Alto, CA 94306 USA}
\icmlauthor{Your CoAuthor's Name}{email@coauthordomain.edu}
\icmladdress{Their Fantastic Institute,
            27182 Exp St., Toronto, ON M6H 2T1 CANADA}

\icmlkeywords{boring formatting information, machine learning, ICML}

\vskip 0.3in
]

\begin{abstract} 
We discuss structured Schatten norms for tensor
decomposition that includes two recently  proposed
norms (``overlapped'' and ``latent'') for convex-optimization-based tensor decomposition, and connect tensor
decomposition with wider literature on structured sparsity. Based on
the properties of the structured Schatten norms, we mathematically
 analyze the performance of ``latent'' approach for tensor decomposition, which
 was empirically found to perform better than the ``overlapped'' approach
in some settings. We show theoretically that this is indeed
the case. In particular, when the unknown true tensor is low-rank in a
 specific mode, this approach performs as good as knowing the mode
 with the smallest rank. Along the way, we show a novel duality result
 for structures Schatten norms, establish the consistency, and discuss
 the identifiability of this approach. We confirm through
 numerical simulations that our theoretical prediction can precisely
 predict the scaling behaviour of 
 the mean squared error. 
\end{abstract} 
\bibliographystyle{icml2013}

\section{Introduction}
Decomposition of tensors~\cite{KolBad09} (or multi-way arrays) into
low-rank components
arises naturally in many real world data analysis problems. For example,
in neuroimaging, we are often interested in finding spatio-temporal
patterns of neural activities
that are related to certain experimental conditions
or subjects; one way to do this is to compute the decomposition of the
data tensor, which can be of size channels $\times$ time-points $\times$ subjects
$\times$ conditions~\cite{Moe11}. In computer vision, an ensemble of face images
can be collected into a tensor of size pixels $\times$ subjects $\times$
illumination $\times$ viewpoints; the decomposition of this tensor
yields the so called {\em tensorfaces}~\cite{VasTer02}, which can be
regarded as a multi-linear generalization of {\em eigenfaces}~\cite{SirKir87}.

\begin{figure}[tb]
 \begin{center}
  \includegraphics[width=\columnwidth,clip]{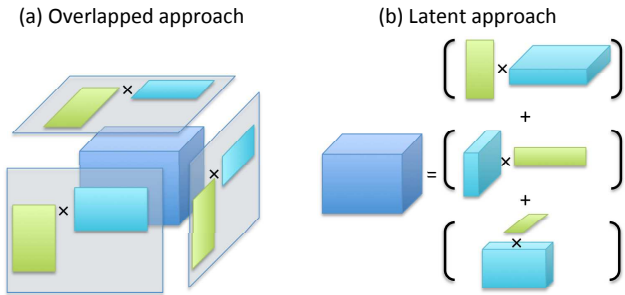}
\caption{Schematic illustrations of the overlapped approach and the latent
  approach for the decomposition of a three way tensor ($K=3$).}
  \label{fig:illust}
 \end{center}
\end{figure}

Conventionally tensor decomposition has been tackled through non-convex
optimization problems, using alternate least squares or higher order
orthogonal iteration~\cite{DeLDeMVan00}. Although being successful in
many application areas, the statistical performance of such approaches has been
widely open. Moreover, the model selection problem can be highly
challenging, especially for the so called Tucker
model~\cite{Tuc66,DeLDeMVan00}, because we 
need to specify the rank $r_k$ for each mode (here a mode refers to one
dimensionality of a tensor); that is, we have $K$ hyper-parameters to choose for a $K$-way
tensor, which is challenging even for $K=3$.  

Recently a convex-optimization-based approach for tensor
decomposition has been proposed by several authors
\cite{SigDeLSuy10,GanRecYam11,LiuMusWonYe09,TomHayKas11}, and
its performance has been analyzed in \cite{TomSuzHayKas11}.

The basic idea behind their convex approach, which we call {\em overlapped
approach}, is to unfold\footnote{For a $K$-way tensor, there are
$K$ ways to unfold a tensor into a matrix. See \Secref{sec:notation}.} a
tensor into matrices along different modes and penalize the unfolded
matrices to be {\em simultaneously low-rank} based on the Schatten
1-norm, which is also known as the trace norm and nuclear
norm~\cite{FazHinBoy01,SreRenJaa05,RecFazPar10}; see the left panel of Figure~\ref{fig:illust}.
The convex approach does not require the rank of the decomposition to be
specified beforehand, and due to the low-rank inducing property of the
Schatten 1-norm, the rank of the decomposition is {\em automatically}
determined.

However, it has been noticed that the above overlapped approach has a
limitation that it performs poorly for a tensor that is only low-rank in
a certain mode \citep{TomHayKas11}. They proposed an alternative approach, which we
call {\em latent approach}, that decomposes a given tensor into a  
a mixture of tensors that each are low-rank in a specific mode;
see the right panel of Figure~\ref{fig:illust}.  
Figure~\ref{fig:overlap_vs_latent} demonstrates that the latent approach
is preferable to the overlapped approach when the underlying tensor is
almost full rank in all but one mode.
\begin{figure}[tb]
 \begin{center}
  \includegraphics[width=\columnwidth,clip]{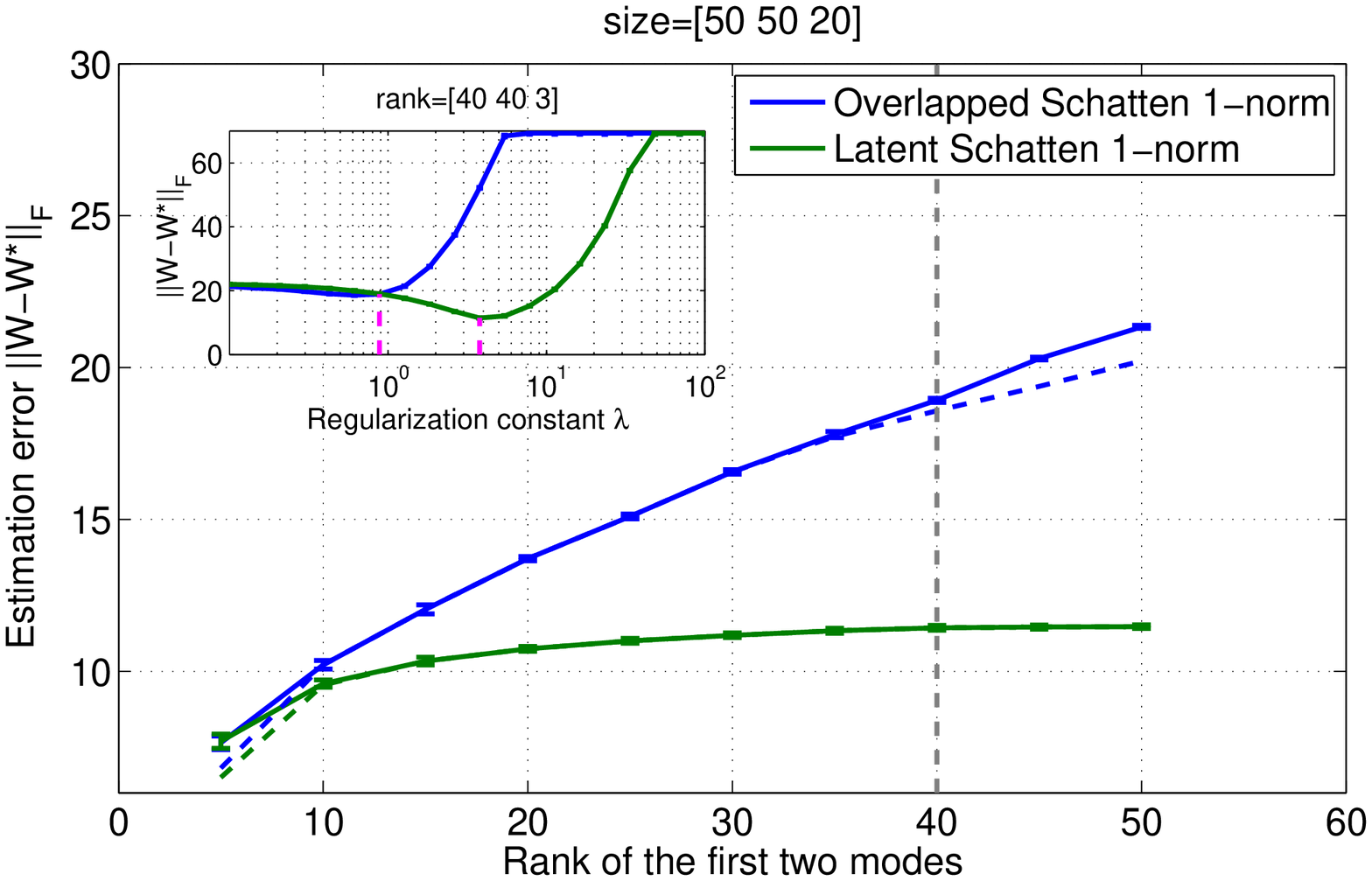}
  \caption{Estimation of a low-rank 50$\times$50$\times$20 tensor of
  rank $r\times r\times 3$ from noisy measurements. The noise standard
  deviation is $\sigma=0.1$. The estimation
  errors of two convex optimization based methods are plotted against
  the rank $r$ of the first two modes. The solid lines show the
error at the fixed regularization constant $\lambda$, which is 0.89
  for the overlapped approach and 3.79 for the latent approach
 (see also Figure~\ref{fig:comparison}).
The dashed lines show the minimum error over candidates
of  the regularization constant $\lambda$  from 0.1 to 100. 
In the inset, the errors of the two approaches are
  plotted against the regularization constant $\lambda$ for rank
 $r=40$ (marked with gray dashed vertical line in the
  outset). The two values (0.89 and 3.79) are marked with vertical dashed lines.
   Note that both approaches need no knowledge of the true rank; the
  rank is automatically learned. }
  \label{fig:overlap_vs_latent}
 \end{center}
\end{figure}


However, there are two issues that are not properly addressed so far.

The first issue is the statistical performance of the latent approach. 
In this paper, we show that the mean squared error of the
latent approach scales no greater than the minimum mode-$k$ rank of the
underlying true tensor, which clearly explains why the latent approach
suffers less than the overlapped approach in Figure~\ref{fig:overlap_vs_latent}.

The second issue is the identifiability of the model underlying the latent
approach, i.e., a mixture of low-rank tensors. In this paper, we show that such
a mixture is identifiable only when the mixture consists of one
component; in other words, when the underlying tensor is low-rank in a
specific mode.


Along the way, we show a novel {\em duality} between the two types of norms employed in
the above two approaches, namely the overlapped Schatten norm and the latent
Schatten norm. This result is closely related and generalize the results
in structured sparsity
literature~\cite{BacJenMaiObo11,JenAudBac11,OboJacVer11,MauPon11}. In
fact, the {\em (plain) overlapped group lasso} constrains the weights to be
simultaneously group sparse over overlapping groups. The {\em latent group lasso} predicts with a mixture of group sparse
weights \citep[see also][]{WriGanRaoPenMa09,JalRavSanRua10,AgaNegWai11}.
These approaches clearly correspond to the two variations of tensor
decomposition algorithms we discussed above.

Finally we empirically compare the overlapped approach and latent approach
and show that even when the unknown tensor is simultaneously low-rank,
which is a favorable situation for the overlapped approach, the latent
approach performs better in many cases. 
Thus we provide both theoretical and empirical
evidence that for noisy tensor decomposition, the latent approach is
preferable to the overlapped approach. 
Our result is  complementary to the
previous study~\cite{TomHayKas11,TomSuzHayKas11}, which mainly focused on
the noise-less tensor completion setting.

This paper is structured as follows. In \Secref{sec:notation}, we
provide basic definitions of the two variations of structured Schatten
norms, namely the overlapped/latent Schatten
norms, and discuss their properties, especially the {\em duality}
between them. \Secref{sec:theory} presents our main theoretical
contributions; we establish the consistency of the latent approach, we
show a denoising performance bound, and discuss the identifiability of
the model underlying it. In \Secref{sec:simulation}, we
empirically confirm the scaling predicted by our theory. Finally,  \Secref{sec:conclusion} concludes
the paper.

\section{Structured Schatten norms for tensors}
\label{sec:notation}
In this section, we define the overlapped Schatten norm and the latent
Schatten norm and discuss their basic properties. 

First we need some basic definitions.

Let $\tW\in\RR^{n_1\times\cdots n_K}$ be a $K$-way tensor. We denote the
total number of entries in $\tW$ by $N=\prod_{k=1}^{K}n_k$. The dot
product between two tensors $\tW$ and $\tX$ is defined as
$\dot{\tW}{\tX}={\rm vec}(\tW)\T{\rm vec}(\tX)$; i.e., the dot product
as vectors in $\RR^{N}$. The Frobenius norm of a tensor is defined as $\norm{\tW}_F=\sqrt{\dot{\tW}{\tW}}$.
Each dimensionality of a tensor is called a {\em mode}. The mode $k$
{\em unfolding} $\mW_{(k)}\in\RR^{n_k\times N/n_k}$ is a matrix that is 
obtained by concatenating the mode-$k$ fibers along columns; here
a mode-$k$ fiber is an $n_k$ dimensional vector obtained by fixing all
the indices but the $k$th index of $\tW$. 
The mode-$k$ rank $\tr_k$ of $\tW$ is the rank of the mode-$k$ unfolding
$\mX_{(k)}$. We say that a tensor $\tW$ has Tucker rank
$(\tr_1,\ldots,\tr_K)$ if the mode-$k$ rank is $\tr_k$ for
$k=1,\ldots,K$~\cite{KolBad09}. 
The mode $k$ folding is the inverse of the unfolding operation.
\subsection{Overlapped Schatten norms}
The low-rank inducing norm studied in
\cite{SigDeLSuy10,GanRecYam11,LiuMusWonYe09,TomHayKas11}, which we call
 overlapped Schatten 1-norm, can be written as follows:
\begin{align}
\label{eq:norm_overlap_S1}
\norm{\tW}_{\underline{S_1/1}}&=\sum\nolimits_{k=1}^{K}\|\mW_{(k)}\|_{S_1}.
\end{align}

In this paper, we consider the following more general {\em overlapped
$S_p/q$-norm}, which includes the Schatten 1-norm as the special case
$(p,q)=(1,1)$. The overlapped $S_p/q$-norm is written as follows:
\begin{align}
\label{eq:norm_overlap}
\norm{\tW}_{\underline{S_p/q}} &=\Bigl(\sum\nolimits_{k=1}^{K}\|\mW_{(k)}\|_{S_p}^q\Bigr)^{1/q},
\end{align}
where $1\leq p,q\leq \infty$; here
$$ \|\mW\|_{S_p}=\Bigl(\sum\nolimits_{j=1}^{r}\sigma_j^p(\mW)\Bigr)^{1/p}$$
is the Schatten $p$-norm for matrices,
where $\sigma_j(\mW)$ is the $j$th largest singular value of $\mW$.

When used as a regularizer, the overlapped Schatten 1-norm
penalizes all modes of $\tW$ to be jointly low-rank. It is related to
the overlapped group regularization \citep[see
][]{JenAudBac11,MaiJenOboBac11} in a
sense that the same object $\tW$ appears repeatedly in the norm.

The following inequality relates the overlapped Schatten 1-norm 
 with the Frobenius norm, which was a
 key step in the analysis of \citet{TomSuzHayKas11}:
 \begin{align}
\label{eq:bound_overlap_fro}
  \norm{\tW}_{\underline{S_1/1}}&\leq \sum_{k=1}^{K}\sqrt{\tr_k}\norm{\tW}_F,
 \end{align}
where  $\tr_k$ is the mode-$k$ rank of $\tW$.

Now we are interested in the dual norm of the overlapped
$S_p/q$-norm, because deriving the dual norm is a key step in solving
the minimization problem that involves the norm~\eqref{eq:norm_overlap}
\citep[see][]{MaiJenOboBac11},
as well as
computing
various complexity measures, such as, Rademacher
complexity~\cite{FoySre11} and Gaussian width~\cite{ChaRecParWil10}. It
turns out that the dual norm of the overlapped $S_p/q$-norm is the  
{\em latent $S_{p^\ast}/q^\ast$-norm} as shown in the following lemma.

\begin{lemma}
\label{lem:dualnorm}
 The dual norm of the overlapped $S_{p}/q$-norm is the latent
 $S_{p^\ast}/q^{\ast}$-norm, where $1/p+1/p^\ast=1$ and
 $1/q+1/q^\ast=1$, which is defined as follows:
\begin{align}
\label{eq:norm_latent}
\norm{\tX}_{\overline{S_{p^\ast}/q^{\ast}}}= \inf_{\left(\tX^{(1)}+\cdots+\tX^{(K)}\right)=\tX}&\left(\sum\nolimits_{k=1}^{K}\|\mX_{(k)}^{(k)}\|_{S_{p^{\ast}}}^{q^\ast}\right)^{1/q^\ast}.
\end{align}
Here the infimum is taken over the $K$-tuple of tensors
 $\tX^{(1)},\ldots,\tX^{(K)}$ that sums to $\tX$.
\end{lemma}
\begin{proof}
 The proof is presented in Appendix~\ref{sec:proof_lemma_dualnorm}.
\end{proof}

The duality in the above lemma naturally generalizes the
duality between overlapped/latent group sparsity norms that have only
partial overlap (in contrast to the complete overlap here). Although
being recognized in special instances
\cite{JalRavSanRua10,OboJacVer11,MauPon11,AgaNegWai11}, to the best of our knowledge,
this duality has not been presented in the generality of Lemma~\ref{lem:dualnorm}.
Note that when the groups have no overlap, the overlapped/latent
group sparsity norms become identical, and the duality is the ordinary
duality between the group $S_p/q$-norms and the group $S_{p^\ast}/q^{\ast}$-norms.

\subsection{Latent Schatten norms}
The latent approach for tensor decomposition
 proposed by \citet{TomHayKas11} solves the following minimization problem
\begin{align}
\label{eq:opt_latent}
 \minimize_{\tW^{(1)},\ldots,\tW^{(K)}}\quad & L(\tW^{(1)}+\cdots +\tW^{(K)})+\lambda\sum_{k=1}^{K}\|\mW_{(k)}^{(k)}\|_{S_1},
\end{align}
where $L$ is a loss function, $\lambda$ is a regularization
constant, and $\mW^{(k)}_{(k)}$ is the mode-$k$ unfolding of $\tW^{(k)}$.
 Intuitively speaking, the latent approach for tensor decomposition predicts with a mixture of $K$ tensors that each
are regularized to be low-rank in a specific mode.

Now, since the loss term in the minimization
 problem~\eqref{eq:opt_latent} only depends on the sum of the tensors
 $\tW^{(1)},\ldots,\tW^{(K)}$, minimization
 problem~\eqref{eq:opt_latent} is equivalent to the following minimization
 problem
\begin{align*}
 \minimize_{\tW}\quad & L(\tW)+\lambda\norm{\tW}_{\overline{S_1/1}}.
\end{align*}
In other words, we have identified the structured Schatten
norm employed in the latent approach as the latent $S_1/1$-norm
(or latent Schatten 1-norm for short), which can be written as follows:
\begin{align}
 \label{eq:norm_latent_S1}
\norm{\tW}_{\overline{S_1/1}}=\inf_{\left(\tW^{(1)}+\cdots+\tW^{(K)}\right)=\tW}\sum_{k=1}^{K}\|\mW_{(k)}^{(k)}\|_{S_1}.
\end{align}
According to Lemma~\ref{lem:dualnorm}, the dual norm of the latent $S_1/1$-norm
is the overlapped $S_\infty/\infty$-norm
\begin{align}
 \label{eq:norm_overlap_Sinf}
\norm{\tX}_{\underline{S_\infty/\infty}}=\max_k\|\mX_{(k)}\|_{S_\infty},
\end{align}
where $\|\cdot\|_{S_\infty}$ is the spectral norm. 

The following lemma is similar to
inequality~\eqref{eq:bound_overlap_fro} and is a key in our analysis.
\begin{lemma}
\label{lem:latent_bound}
\begin{align*}
 \norm{\tW}_{\overline{S_1/1}}&\leq \left(\min_k \sqrt{\tr_k}\right)\norm{\tW}_F,
\end{align*} 
where $\tr_k$ is the mode-$k$ rank of $\tW$.
\end{lemma}
\begin{proof}
Since we are allowed to take a singleton decomposition $\tW^{(k)}=\tW$
and $\tW^{(k')}=0$ $(k'\neq k)$, we have
\begin{align*}
\norm{\tW}_{\overline{S_1/1}}&=\inf_{\left(\tW^{(1)}+\cdots+\tW^{(K)}\right)=\tW}\sum_{k=1}^{K}\|\mW_{(k)}^{(k)}\|_{S_1}\\
&\leq \|\mW_{(k)}\|_{S_1}\\
&\leq \sqrt{\tr}_k\|\mW_{(k)}\| \quad(\forall k=1,\ldots,K)
\end{align*}
Choosing $k$ that minimizes the right hand side, we obtain our claim.
\end{proof}
Compared to inequality~\eqref{eq:bound_overlap_fro}, the latent Schatten
1-norm is bounded by the {\em minimal} square root of the ranks instead of the
sum. This is the fundamental reason why the latent approach performs
betters than the overlapped approach as in Figure~\ref{fig:overlap_vs_latent}.

%

\section{Main theoretical results}
\label{sec:theory}
In this section, we study the consistency, generalization performance,
and identifiability of the latent approach for 
tensor decomposition in the context of 
recovering an unknown tensor $\tW^{\ast}$ from noisy measurements.
This is the setting of the
experiment in Figure~\ref{fig:overlap_vs_latent}. 

First, we show that the latent approach is consistent. That is, the error
goes to zero when the noise goes to zero, which corresponds to the
situation when the entries are repeatedly observed.

Second, combining the duality we presented
in the previous section with the techniques from \citet{AgaNegWai11}, we
analyze the denoising performance of the latent approach in the context of
recovering an unknown tensor $\tW^{\ast}$ from noisy measurements.
This is the setting of the
experiment in Figure~\ref{fig:overlap_vs_latent}. 
We first prove a
deterministic inequality that holds under certain condition on the
regularization constant. Next, we assume Gaussian noise and derive
an inequality that holds with high probability under an appropriate
scaling of the regularization constant.

Third, we discuss the
difference between overlapped approach and latent approach and provide
an explanation for the empirically observed superior performance of the
latent approach in Figure~\ref{fig:overlap_vs_latent}.

Finally we discuss the condition under which the decomposition
$\tW=\sum_{k=1}^{K}\tW^{(k)}$ is identifiable and show  that the model
is (locally) identifiable only when the mixture consists of one component. 


\subsection{Consistency}
\label{sec:consistency}
Let $\tW^{\ast}$ be the underlying true tensor and the noisy version
$\tY$ is obtained as follows:
\begin{align*}
\tY &= \tW^{\ast} + \tE,
\end{align*}
where $\tE\in\RR^{n_1\times \cdots\times n_K}$ is the noise tensor.

First we establish the consistency of the latent approach.
\begin{theorem}
 \label{thm:consistency}
The estimator defined by
\begin{align}
\label{eq:opt_basic}
 \hat{\tW}=\argmin_{\tW}\Biggl(
\frac{1}{2}&\norm{\tY- \tW}_F^2+\lambda\norm{\tW}_{\overline{S_1/1}}\Biggr),
\end{align}
is consistent. That is, when the noise goes to zero (e.g., when the
 entries are repeatedly observed),
$\hat{\tW}\rightarrow \tW^{\ast}$ for any sequence $\lambda\rightarrow 0$.
\end{theorem}

\begin{proof}
Due to the triangular inequality 
\begin{align*}
 \norm{\hat{\tW}-\tW^{\ast}}_F&\leq\norm{\hat{\tW}-\tY}_F+\norm{\tY-\tW^{\ast}}_F.
\end{align*}
Here the second term goes to zero as the noise shrinks. Next, from the
optimality of $\hat{\tW}$, the first term satisfies
\begin{align*}
 \tY - \hat{\tW}&\in\lambda\partial\norm{\hat{\tW}}_{\overline{S_1/1}},
\end{align*}
where $\partial\norm{\hat{\tW}}_{\overline{S_1/1}}$ is the
 subdifferential of the latent $S_1/1$ norm at $\hat{\tW}$. 
Now since the dual norm of
the latent $S_1/1$ norm is the overlapped $S_\infty/\infty$ norm, for
 any $\tG\in\partial\norm{\hat{\tW}}_{\overline{S_1/1}}$, we
have 
 $\norm{\tG}_{\underline{S_\infty/\infty}}\leq 1$, and therefore
\begin{align*}
 \norm{\hat{\tW}-\tY}_F&\leq C\norm{\hat{\tW}-\tY}_{\underline{S_\infty/\infty}}\leq C\lambda,
\end{align*}
where $C$ is a constant that is independent of $\lambda$. Therefore, for
any sequence $\lambda\rightarrow 0$, we have $\hat{\tW}\rightarrow
\tW^{\ast}$ when $\tE\rightarrow 0$.
\end{proof}

\subsection{Deterministic bound}
\label{sec:detbound}
The consistency statement in the previous section only deals with the
sum $\hat{\tW}=\sum_{k=1}^{K}\hat{\tW}^{(k)}$ and its convergence to the
truth $\tW^{\ast}$ in the limit the noise goes to zero. In this section, we
establish a stronger statement that shows the behavior of individual
terms $\hat{\tW}^{(k)}$ and also the denoising performance.

To this end we need some additional assumptions.

First, we assume that the unknown tensor $\tW^{\ast}$ is a mixture of $K$
tensors that each are low-rank in a certain mode and we have
a noisy observation $\tY$  as follows:
\begin{align}
\label{eq:noise_model}
\tY=\tW^{\ast}+\tE=\sum\nolimits_{k=1}^{K}\tW^{\ast(k)}+\tE,
\end{align}
where $\bar{r}_k={\rm rank}(\mW^{(k)}_{(k)})$ is the
mode-$k$ rank of the $k$th component $\tW^{\ast(k)}$.

Second, we assume that the spectral norm of the mode-$k$ unfolding of the $l$th
component is bounded by a constant $\alpha$ for all $k\neq l$ as follows:
\begin{align}
 \label{eq:assumption_Sinf}
\|\mW_{(k)}^{\ast(l)}\|_{S_\infty}&\leq \alpha \quad (\forall l\neq k, k,l=1,\ldots,K).
\end{align}
Note that such an additional incoherence assumption has also been used in
\cite{CanLiMaWri09,WriGanRaoPenMa09,AgaNegWai11,HsuKakZha11}. 

We employ the following optimization problem to recover the unknown
tensor $\tW^{\ast}$:
\begin{align}
\label{eq:opt}
 \hat{\tW}=\argmin_{\tW}\Biggl(
\frac{1}{2}&\norm{\tY- \tW}_F^2+\lambda\norm{\tW}_{\overline{S_1/1}}\notag\\
&\quad {\rm  s.t.}\quad \|\mW^{(k)}_{(l)}\|_{S_\infty}\leq \alpha, \quad
 \forall l\neq k\Biggr),
\end{align}
where $\tW=\sum_{k=1}^{K}\tW^{(k)}$ denotes the
optimal decomposition induced by the latent Schatten 1-norm~\eqref{eq:norm_latent_S1}; $\lambda>0$ is a regularization
constant. Notice that we have introduced additional spectral norm
constraints to control the correlation between the components
 \citep[see also][]{AgaNegWai11}. 

Our first bound can be stated as follows:
\begin{theorem}
\label{thm:deterministic}
Let $\hat{\tW}^{(k)}$ be an optimal decomposition of $\hat{\tW}$ induced by the latent
 Schatten 1-norm~\eqref{eq:norm_latent_S1}.
Assume that the regularization constant $\lambda$ satisfies
 $\lambda\geq 2\norm{\tE}_{\underline{S_\infty/\infty}}+\alpha(K-1)$. Then
there is a universal constant $c$ such that, any solution $\hat{\tW}$ of the
 minimization problem~\eqref{eq:opt} satisfies the
 following deterministic bound:
\begin{align}
\label{eq:thm_det_statement}
 \sum_{k=1}^{K}\norm{\hat{\tW}^{(k)}-\tW^{\ast(k)}}_F^2 &\leq c\lambda^2\sum_{k=1}^{K}\lr_k.
\end{align}
\end{theorem}
\begin{proof}
 The proof is presented in Appendix~\ref{sec:proof_deterministic}.
\end{proof}

We can also obtain a bound on the difference of the whole tensor
$\hat{\tW}-\tW^{\ast}$ rather than the squared sum differences as in
Theorem~\ref{thm:deterministic} as follows.
\begin{corollary}
Under the same conditions as in Theorem~\ref{thm:deterministic}  we have
\begin{align}
\label{eq:bound_main}
 \norm{\hat{\tW}-\tW^{\ast}}_F^2 &\leq cK\lambda^2\sum_{k=1}^{K}\bar{r}_k.
\end{align}
\end{corollary}
\begin{proof}
 Using the triangular inequality and Cauchy-Schwarz inequality we have
 $\norm{\hat{\tW}-\tW^{\ast}}_F\leq\sum_{k=1}^{K}\norm{\hat{\tW}^{(k)}-\tW^{\ast(k)}}_F\leq\sqrt{K}\sqrt{\sum_{k=1}^{K}\norm{\hat{\tW}^{(k)}-\tW^{\ast(k)}}_F^2}$.
\end{proof}

Since we are bounding the overall error in \eqref{eq:bound_main}, we may
exploit the arbitrariness of the decomposition $\tW^{\ast}=\sum_{k=1}^{K}\tW^{\ast(k)}$
to obtain a tight bound. The tightest bound is obtained when we choose
the decomposition that 
minimizes the {\em sum of the ranks} $\sum_{k=1}^{K}\bar{r}_k$. We say
$\tW^{\ast}$ has the {\em latent rank} $(\lr_1,\ldots,\lr_K)$ for such a
minimal decomposition in terms of the sum.

A simple upper bound  is obtained by choosing a decomposition
$\tW^{\ast(k)}=\tW^{\ast}$ and $\tW^{\ast(k')}=0$ for $k'\neq k$.
In particular by choosing the mode with the minimum mode-$k$ rank, we obtain
\begin{align*}
\norm{\hat{\tW}-\tW^{\ast}}_F^2 \leq c K\lambda^2\min_{k=1,\ldots,K}\tr_k,
\end{align*}
where $\tr_k$ is the mode-$k$ rank of $\tW^{\ast}$. We refer to the above
decomposition as the {\em minimum rank singleton decomposition}.

Note that the right-hand side of our bound
\eqref{eq:thm_det_statement} does not necessarily go to zero 
when the noise $\tE$ goes to zero, because
$\lambda\geq\alpha(K-1)$. When the noise goes to zero,
$\hat{\tW}\rightarrow \tW^{\ast}$ can be obtained by any decreasing
sequence  $\lambda\rightarrow 0$ as shown in the previous subsection.
Therefore our bound is most useful when the noise is relatively
large and the first term $2\norm{\tE}_{\underline{S_\infty}/\infty}$
dominates the second term $\alpha(K-1)$ in the condition for the
regularization constant $\lambda$.

\subsection{Gaussian noise}
When the elements of the noise tensor $\tE$ are Gaussian, we obtain
the following theorem.
\begin{theorem}
\label{thm:gaussian}
Assume that the elements of the noise tensor $\tE$ are independent
Gaussian random variables with variance $\sigma^2$. In addition,
 assume without loss of generality that the dimensionalities of
 $\tW^{\ast}$ are sorted in the descending order, i.e.,
 $n_1\geq\cdots\geq n_K$.
Then there are
 universal constants $c_0,c_1$ such that, with high probability, any solution of the
 minimization problem~\eqref{eq:opt} with regularization constant
 $\lambda=c_0\sigma(\sqrt{N/n_K}+\sqrt{n_1}+\sqrt{\log K})+\alpha(K-1)$ satisfies
 the following bound:
\begin{align}
\label{eq:bound_gaussian}
\frac{1}{N} \sum_{k=1}^{K}\norm{\hat{\tW}^{(k)}-\tW^{\ast(k)}}_F^2&\leq
c_1 F\sigma^2\frac{\sum_{k=1}^{K}\bar{r}_k}{n_K},
\end{align}
where $F=\left(\left(1+\sqrt{\frac{n_1n_K}{N}}\right)+\left(\sqrt{\log K}+\frac{\alpha(K-1)}{c_0\sigma}\right)\sqrt{\frac{n_K}{N}}\right)^2$ is a factor that mildly depends on the
 dimensionalities and the constant $\alpha$ in \eqref{eq:assumption_Sinf}.
\end{theorem}
\begin{proof}
 The proof is presented in Appendix~\ref{sec:proof_gaussian}
\end{proof}
Note that the theoretically optimal choice of regularization constant $\lambda$
is independent of the Tucker/latent rank of the truth $\tW^{\ast}$,
which is unknown in practice.

Again we can obtain a bound corresponding to the minimum rank
singleton decomposition as in inequality~\eqref{eq:bound_main} as follows:
\begin{align}
\label{eq:bound_min_decomp}
 \frac{1}{N}\norm{\hat{\tW}-\tW^{\ast}}_F^2&\leq c_1 K
 F\sigma^2\frac{\min_k \tr_k}{n_K},
\end{align}
where $F$ is the same factor as in Theorem~\ref{thm:gaussian}.

\subsection{Comparison with the overlapped approach}
Inequality \eqref{eq:bound_min_decomp} explains the superior performance of the latent approach for
tensor decomposition in Figure~\ref{fig:overlap_vs_latent}. The inequality
obtained in \cite{TomSuzHayKas11} for the overlapped approach that uses
overlapped Schatten 1-norm~\eqref{eq:norm_overlap_S1} can be stated as follows:
\begin{align}
\label{eq:bound_overlap}
 \frac{1}{N}\norm{\hat{\tW}-\tW^{\ast}}_F^2&\leq c_1'\sigma^2\!\!\left(\frac{1}{K}\sum_{k=1}^{K}\sqrt{\textstyle\frac{1}{n_k}}\right)^2\!\!\left(\frac{1}{K}\sum_{k=1}^{K}\sqrt{\tr_k}\right)^2.
\end{align}
Comparing inequalities \eqref{eq:bound_min_decomp} and
\eqref{eq:bound_overlap}, we notice that the complexity of the
overlapped approach depends on the average (square root) of the Tucker rank
$\tr_1,\ldots,\tr_K$, whereas that of the latent approach only grows
linearly against the {\em minimum} Tucker rank. Interestingly, the latent approach
performs {\em as if it knows the mode with the minimum rank}, although
such information is not available to it. However in
inequality~\eqref{eq:bound_min_decomp} we have the factor $K$. This
means that if the mode with the minimum rank is known, the latent
approach looses by constant factor $K$ against the 
simple matrix decomposition approach that unfolds the given tensor at
the minimal rank mode and performs ordinary Schatten 1-norm minimization.


\subsection{Discussion on the identifiability}
\label{sec:identifiability}
Let $\bar{r}_k={\rm rank}(\mW^{(k)}_{(k)})$ be the mode-$k$ rank of the
 $k$th component $\tW^{(k)}$ in the decomposition
\begin{align}
\label{eq:model}
 \tW & = \tW^{(1)}+\tW^{(2)}+\cdots+\tW^{(K)}.
\end{align}
 We say that a decomposition \eqref{eq:model} is {\em locally 
identifiable} when there is no other decomposition
$\sum_{k=1}^{K}\tilde{\tW}^{(k)}$ having the same rank $(\bar{r}_1,\ldots,\bar{r}_K)$.
The following theorem fully characterizes the local identifiability of the decomposition~\eqref{eq:model}.

\begin{theorem}
\label{thm:identifiability}
 The decomposition~\eqref{eq:model} is {\em locally identifiable} if and
only if $\tW^{(k^{\ast})}=\tW$ for $k=k^{\ast}$ and $\tW^{(k)}=0$ otherwise, for some $k^{\ast}$.
\end{theorem}
\begin{proof}
 The proof is given in Appendix~\ref{sec:proof_identifiability}.
\end{proof}
The above theorem partly explains the difficulty of estimating
individual components $\tW^{\ast(k)}$ {\em without additional incoherence
assumption} as in~\eqref{eq:assumption_Sinf}. In fact, most
decompositions of the form \eqref{eq:noise_model} are not identifiable. 


\section{Numerical results}
\label{sec:simulation}
In this section, we numerically confirm the scaling behavior we have
theoretically predicted in the last section.

The goal of this experiment is to recover the true low rank tensor
$\tW^{\ast}$ from a noisy observation $\tY$.
We randomly generated the true low rank tensors $\tW^{\ast}$ of size
$50\times 50\times 20$ or 
$80\times 80\times 40$ with various Tucker ranks
$(\tr_1,\tr_2,\tr_3)$. A low-rank tensor is generated by first randomly drawing the
$\tr_1\times\tr_2\times \tr_3$ core tensor from the standard normal
distribution and multiplying an orthogonal factor matrix drawn from the
Haar measure to its each mode. The observation tensor $\tY$ is obtained
by adding Gaussian noise with standard deviation $\sigma=0.1$. There is
no missing entries in this experiment. 

\begin{figure*}[tb]
 \begin{center}
  \includegraphics[width=\textwidth]{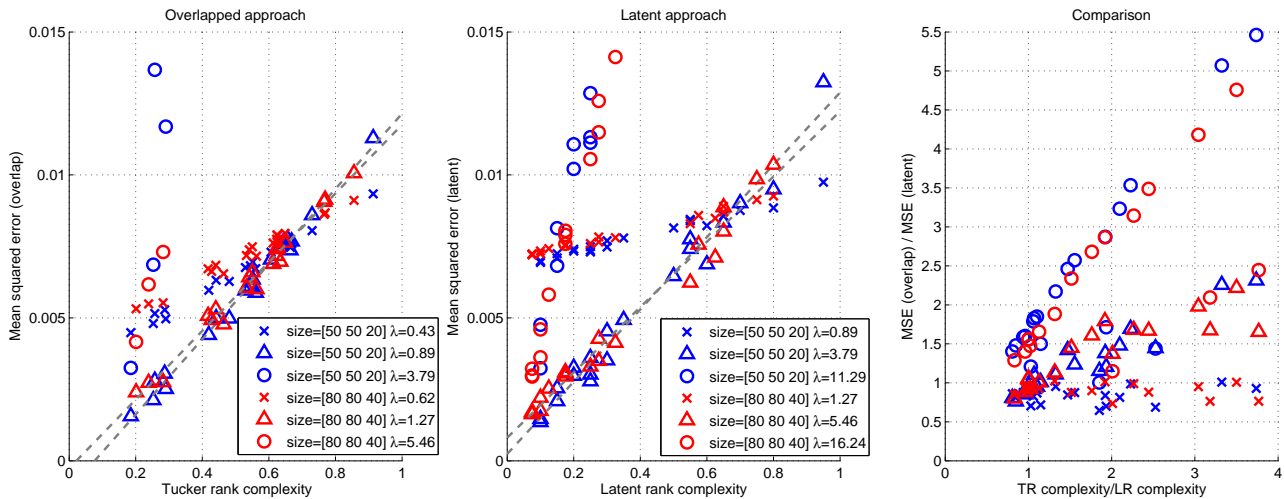}
  \caption{Performance of the overlapped approach and latent approach
  for tensor decomposition are shown against their theoretically
  predicted complexity measures (see Eqs.~\eqref{eq:trc} and
  \eqref{eq:lrc}). The right panel shows the improvement of the latent
  approach from the overlapped approach against the ratio of their
  complexity measures.}
  \label{fig:comparison}
 \end{center}
\end{figure*}

For an observation $\tY$, we computed tensor decompositions using
the overlapped approach and the latent
approach~\eqref{eq:opt} using the solver available from the webpage\footnote{\url{http://www.ibis.t.u-tokyo.ac.jp/RyotaTomioka/Softwares/Tensor}} of
one of the authors of \citet{TomHayKas11}. The solver uses the
alternating direction method of multipliers~\citep{GabMer76} and the algorithm is
described in the above paper. We computed the solutions
for 20 candidate regularization constants ranging from 0.1 to 100 and
report the results for three representative values for each method.

We measured the quality of the solutions obtained by the two approaches
by the mean squared error (MSE) $\norm{\hat{\tW}-\tW^{\ast}}_F^2/N$.  In
order to make our theoretical predictions more concrete, we
define the quantities in the right 
hand side of the bounds
\eqref{eq:bound_overlap} and  \eqref{eq:bound_gaussian}  as {\em Tucker
rank (TR) complexity} and  {\em Latent rank (LR) complexity}, respectively, as follows:
\begin{align}
\label{eq:trc}
\textrm{TR complexity}&=\left(\textstyle\frac{1}{K}\sum\nolimits_{k=1}^{K}\sqrt{\textstyle\frac{1}{n_k}}\right)^2\left(\textstyle\frac{1}{K}\sum\nolimits_{k=1}^{K}\sqrt{\tr_k}\right)^2,\\
\label{eq:lrc}
 \textrm{LR complexity}&=\frac{\sum_{k=1}^{K}\bar{r}_k}{n_K},
\end{align}
where without loss of generality we assume $n_1\geq\cdots\geq n_K$. We have
ignored terms like $\sqrt{n_k/N}$ because they are negligible for
$n_k\approx 50$ and $N\approx 50,000$. 
The TR complexity is equivalent
to the {\em normalized rank} in \cite{TomSuzHayKas11}.
Note that the TR complexity~\eqref{eq:trc} is defined in terms of
the Tucker rank $(\tr_1,\ldots,\tr_K)$ of the truth $\tW^{\ast}$,
whereas the LR complexity~\eqref{eq:lrc} is defined in terms of the
latent rank $(\lr_1,\ldots,\lr_K)$ (see \Secref{sec:detbound}). In order
to compute the sum of latent ranks $\sum_{k=1}^{K}\lr_k$, we ran the latent 
approach to the true tensor $\tW^{\ast}$ without noise, and took the minimum
of the sums obtained from that and the minimum rank singleton
decomposition. 
The whole procedure is repeated 10 times and averaged. 

Figure~\ref{fig:comparison} shows the results of the experiment. The
left panel shows the MSE of the overlapped approach
against the TR complexity~\eqref{eq:trc}. The middle panel shows the
MSE of the latent approach against the LR
complexity~\eqref{eq:lrc}. The right panel shows the improvement (i.e.,
MSE of the overlap approach divided by that of the latent approach)
against the ratio of the respective complexity measures.

First, from the left panel we can confirm that as predicted by
\cite{TomSuzHayKas11}, the MSE of the overlapped approach scales
linearly against the TR complexity~\eqref{eq:trc} for each value of the
regularization constant. We can also see that as predicted by
Theorem~\ref{thm:gaussian}, by scaling the
regularization constant proportionally with $\sqrt{N/n_K}$, the series
corresponding to size $50\times 50\times 20$ and those corresponding to size
$80\times 80\times 40$ almost lie on top of each others.

From the central panel, we can clearly see that the MSE of the
latent approach scales linearly against the LR complexity~\eqref{eq:lrc}
as predicted by Theorem~\ref{thm:gaussian}.
 The series with $\bigtriangleup$ ($\lambda=3.79$ for $50\times 50\times
 20$, $\lambda=5.46$ for $80\times 80 \times 40$) is
mostly below other series, which means that the optimal choice of the
regularization constant is independent of the rank of the true tensor
and only depends on the size; this agrees with the condition on
$\lambda$ in Theorem~\ref{thm:gaussian}. Since the blue series and red
series with the same markers lie on top of each other (especially 
the series with $\bigtriangleup$ for which the optimal regularization
constant is chosen),
we can see that our theory predicts not only the scaling against the latent
ranks but also that against the size of the tensor correctly.
Note that the regularization constants are scaled by roughly 1.6 to
account for the difference in the dimensionality.

The right panel reveals that in many cases the latent approach performs
better than the overlapped approach, i.e., MSE (overlap)/ MSE (latent)
greater than one. Moreover, we can see
that the success of the latent approach relative to the overlapped
approach is correlated with high TR complexity to LR complexity
ratio. Indeed, we found that the optimal decomposition of the true
tensor $\tW^{\ast}$ was typically a singleton decomposition
corresponding to the smallest tucker rank (see \Secref{sec:detbound}).

One might think that we can fix the overlapped approach by allowing
individual regularization constant for each mode. However, this would
only be possible if we knew the mode with small rank.

The improvements here are milder than that in
Figure~\ref{fig:overlap_vs_latent}. This is because most of the randomly
generated low-rank tensors were simultaneously low-rank to some
degree. It is interesting that the latent approach perform at
least as good as the overlapped approach also in such situations.

\section{Conclusion}
\label{sec:conclusion}
In this paper, we have presented a framework for structured Schatten
norms. The current framework includes both the overlapped Schatten
1-norm and latent Schatten 1-norm recently proposed in the context of
convex-optimization-based tensor
decomposition~\citep{SigDeLSuy10,GanRecYam11,LiuMusWonYe09,TomHayKas11}, and
connects these studies to the broader studies on
structured sparsity~\citep{BacJenMaiObo11,JenAudBac11,OboJacVer11,MauPon11}.
Moreover, we have shown a {\em duality} that holds between the two types
of norms.

Furthermore, we have rigorously studied the performance of the latent approach for tensor
decomposition. We have shown the consistency of the latent Schatten
1-norm minimization. Next, we have analyzed the denoising performance of the
latent approach and shown that the error of the latent
approach is upper bounded by the {\em minimum} Tucker rank, which
contrasts sharply against the average (square root) dependency of the
overlapped approach analyzed in \citet{TomSuzHayKas11}. This explains the
empirically observed superior performance of 
the latent approach compared to the overlapped approach.
 The most difficult case for the overlapped approach is when the unknown
 tensor is only low-rank in one mode as in Figure~\ref{fig:overlap_vs_latent}.

We have also confirmed through numerical simulations that our analysis precisely
predicts the scaling of the mean squared error as a function of the
dimensionalities and the latent rank of the unknown tensor. Unlike 
Tucker rank, latent rank of a tensor is not easy to
compute. However, note that the theoretically optimal scaling of the
regularization constant does not depend on the latent rank.

Therefore we have theoretically and empirically shown
that for noisy tensor decomposition, the latent approach is more likely to
perform better than the overlapped approach. Analyzing the performance
of the latent approach for tensor completion would be an important
future work. 

The structured Schatten norms proposed in this paper include norms for
tensors that are not employed in practice yet. Therefore, we envision
that this paper serve as a starting point for various extensions, e.g.,
using the overlapped $S_1/\infty$-norm instead of the $S_1/1$-norm 
or a {\em non-sparse} tensor decomposition similar to the $\ell_p$-norm
MKL~\cite{MicPon05,KloBreSonZie11}.

{\small \bibliography{icml2013}}

\clearpage
\appendix
\noindent{\bf\Large Supplementary material for  ``Convex Tensor Decomposition via Structured Schatten Norms''}

\section{Proof of Lemma~\ref{lem:dualnorm}}
\label{sec:proof_lemma_dualnorm}
\begin{proof}
From the definition, the dual norm
 $\norm{\tX}_{(\underline{S_p/q})^{\ast}}$ can be written as follows:
\begin{align*}
\norm{\tW}_{(\underline{S_p/q})^{\ast}}=\sup \dot{\tW}{\tX}\quad {\rm
 s.t.}\quad \norm{\tW}_{\underline{S_p/q}}\leq 1.
\end{align*}
The basic strategy of the proof is to rewrite the above maximization
problem as a constraint optimization problem and derive the dual
 problem.

First, we rewrite the above maximization problem as follows:
\begin{align*}
\norm{\tX}_{(\underline{S_p/q})^{\ast}}&=\sup
 \frac{1}{K}\sum_{k=1}^{K}\dot{\mZ_k}{\mX_{(k)}} \\
&{\rm s.t.}\quad \mZ_k=\mW_{(k)}, \sum_{k=1}^{K}\|\mZ_k\|_{S_p}^q\leq 1,
\end{align*}
where $\mZ_k\in\RR^{n_k\times N/n_k}$ ($k=1,\ldots,K$) are auxiliary
 variables.

Next we write down the Lagrangian as follows:
\begin{align*}
 L&=\frac{1}{K}\sum_{k}\dot{\mZ_k}{\mX_{(k)}}\\
&\qquad+\frac{1}{K}\sum_{k}\dot{\tilde{\tY}^{(k)}}{\mZ_k^{(k)}-\mW}\\
&\qquad+\frac{\gamma}{Kq}\left(1-\sum\nolimits_{k}\|\mZ_k\|_{S_p}^q\right),
\end{align*}
where $\tilde{\tY}^{(k)}\in\RR^{n_1\times\cdots\times n_K}$
 ($k=1,\ldots,K$), and $\gamma\geq 0$ are Lagrangian multipliers.

Note that for $\mX,\mZ\in\RR^{R\times C}$, we have
\begin{align*}
& \sup_{\mZ}\left(\dot{\mX}{\mZ}-\frac{\gamma}{q}\|\mZ\|_{S_p}^q\right)\\
&\leq\gamma\sup_{\mZ}\left(\|\mX/\gamma\|_{S_{p^\ast}}\|\mZ\|_{S_p}-\frac{1}{q}\|\mZ\|_{S_p}^q\right)\\
&\leq\frac{\gamma^{1-q^\ast}}{q^{\ast}}\|\mX\|_{S_{p^\ast}}^{q^\ast}.
\end{align*}
Here the first equality is achieved if we take $\mZ=c\mU{\rm
 diag}(\sigma_1^{p^\ast/p},\ldots,\sigma_r^{p^\ast/p})\mV\T$, where
$\mU{\rm diag}(\sigma_1,\ldots,\sigma_r)\mV\T$ is the
 singular value decomposition of the matrix $\mX/\gamma$,  and $c$ is an
 arbitrary scaling constant. 
 The second equality is achieved if we take $\|\mZ\|_{S_p}=\|\mX/\gamma\|_{S_{p^\ast}}^{\frac{1}{q-1}}$.

Thus, maximizing the Lagrangian with respect to $\mZ_k$ ($k=1,\ldots,K$)
 and $\tW$, we obtain the dual problem
\begin{align*}
 \norm{\tX}_{\left(S_1/q\right)^{\ast}}&=\!\!\!\!\inf_{\gamma,\tY^{(1)},\ldots,\tY^{(K)}}\!\!\left(\frac{\gamma^{1-q^{\ast}}}{K^{1-q^{\ast}}q^{\ast}}\sum_{k=1}^{K}\|\mY_{(k)}^{(k)}\|_{S_{p^\ast}}^{q^\ast}+\frac{\gamma}{Kq}\right)\\
{\rm s.t.}\quad&\tY^{(1)}+\cdots+\tY^{(K)}=\tX,
\end{align*}
where we used the change of  variable $(\tX+\tilde{\tY}^{(k)})/K=:\tY^{(k)}$.
Furthermore, by explicitly minimizing over $\gamma$, we have
 $\gamma/K=(\sum_{k=1}^{K}\|\mY_{(k)}^{(k)}\|_{S_{p^\ast}}^{q^\ast})^{1/q^{\ast}}$
 and we obtain the statement of the lemma.
\end{proof}


\section{Proof of Theorem~\ref{thm:deterministic}}
\label{sec:proof_deterministic}
Let $\hat{\tW}=\sum_{k=1}^{K}\hat{\tW}^{(k)}$ be the solution and its
optimal decomposition of the minimization problem~\eqref{eq:opt}; in
addition let $\tDelta^{(k)}:=\hat{\tW}^{(k)}-\tW^{\ast(k)}$.

The proof is based on Lemmas \ref{lem:decompose} and
\ref{lem:lowerbound}, which we present below.

In order to present the first lemma, we need the following definitions.
 Let  $\mU_k\mS_k\mV_k=\mW^{\ast(k)}_{(k)}$ be the singular value
 decomposition of the mode-$k$ unfolding of the $k$th component of the unknown tensor $\tW^{\ast}$. We define
 the orthogonal projection of $\Delta^{(k)}$ as follows:
\begin{align*}
 \mDelta^{(k)}_{(k)}&=\mDelta_k' + \mDelta_k'',
\intertext{where}
\mDelta_k'' &= (\mI_{n_k}-\mU_k\mU_k\T)\mDelta_{(k)}^{(k)}(\mI_{N/n_k}-\mV_k\mV_k\T).
\end{align*}
Intuitively speaking, $\mDelta_k''$ lies in a subspace completely
orthogonal to the unfolding of the $k$th component
$\mW^{\ast(k)}_{(k)}$, whereas $\mDelta_k'$ lies in a  partially correlated subspace.

The following lemma is similar to
\citet[Lemma 1]{NegRavWaiYu09} and \citet[Lemma 2]{TomSuzHayKas11}, and
it bounds the Schatten 1-norm of the orthogonal part $\mDelta_k''$ with
that of the partially correlated part $\mDelta_k'$ and also bounds the rank of $\mDelta_k'$ .
\begin{lemma}
 \label{lem:decompose}
Let $\hat{\tW}$ be the solution of the minimization
 problem~\eqref{eq:opt} with the regularization constant $\lambda\geq 2\norm{\tE}_{\underline{S_\infty/\infty}}$.
Let $\tDelta^{(k)}$ and its decomposition be as defined above. Then we have
\begin{enumerate}
 \item ${\rm rank}(\mDelta_k')\leq 2\bar{r}_k$.
 \item $\sum_{k=1}^{K}\|\mDelta_k''\|_{S_1}\leq 3\sum_{k=1}^{K}\|\mDelta_k'\|_{S_1}$.
\end{enumerate}
\end{lemma}
Note that although the proof of the above statement  closely follows that of
\citet[Lemma 2]{TomSuzHayKas11}, the notion of rank is
different. In their result, the rank is the Tucker rank $\tr_k$, whereas
the rank here is the mode-$k$ rank of the $k$th component
$\tW^{\ast(k)}$ of the truth.

The following lemma relates the squared Frobenius norm of the difference
of the sums $\norm{\sum_{k=1}^{K}\tDelta^{(k)}}_F^2$ with the sum of squared
differences $\sum_{k=1}^{K}\norm{\tDelta^{(k)}}_F^2$
\begin{lemma}
 \label{lem:lowerbound}
Let $\hat{\tW}$ be the solution of the minimization
 problem~\eqref{eq:opt}. Then we have,
\begin{align*}
\frac{1}{2} \sum_{k=1}^{K}\norm{\tDelta^{(k)}}_F^2\leq
\frac{1}{2} \norm{\tDelta}_F^2+\alpha (K-1)\sum_{k=1}^{K}\|\mDelta_{(k)}^{(k)}\|_{S_1},
\end{align*}
where $\tDelta=\sum_{k=1}^{K}\tDelta^{(k)}$.
\end{lemma}

\begin{proof}[Proof of Theorem~\ref{thm:deterministic}.]
First from the optimality of $\hat{\tW}$, we have
\begin{align*}
 \frac{1}{2}\norm{\tY-\hat{\tW}}_F^2&+\lambda\sum\nolimits_{k=1}^{K}\|\hat{\mW}^{(k)}_{(k)}\|_{S_1}\\
&\leq\frac{1}{2}\norm{\tY-\tW^{\ast}}_F^2+\lambda\sum\nolimits_{k=1}^{K}\|\mW^{\ast(k)}_{(k)}\|_{S_1},
\end{align*}
which implies
\begin{align}
\frac{1}{2}\norm{\tDelta}_F^2&\leq
 \dot{\tDelta}{\tE}+\lambda\sum\nolimits_{k=1}^{K}\|\mDelta_{(k)}^{(k)}\|_{S_1}\nonumber\\
\label{eq:proof_deterministic_1}
&\leq(\norm{\tE}_{\underline{S_\infty/\infty}}+\lambda)\sum\nolimits_{k=1}^{K}\|\mDelta_{(k)}^{(k)}\|_{S_1},
\end{align}
where we used the fact that $\tY=\tW^{\ast}+\tE$ and the triangular
 inequality in the first line, and H\"older's inequality in the second
 line. Note that there is an additional looseness in the second line due
 to the fact that $\tDelta=\sum_{k=1}^{K}\tDelta^{(k)}$ is not the
 optimal decomposition of $\tDelta$ induced by the latent Schatten 1-norm.

Next, combining inequality \eqref{eq:proof_deterministic_1} with
 Lemma~\ref{lem:lowerbound}, we have
\begin{align}
\label{eq:proof_deterministic_2}
\frac{1}{2}\sum\nolimits_{k=1}^{K}\norm{\tDelta^{(k)}}_F^2 &\leq2\lambda
\sum\nolimits_{k=1}^{K}\|\mDelta_{(k)}^{(k)}\|_{S_1},
\end{align}
where we used the fact that $\lambda\geq \norm{\tE}_{\underline{S_\infty/\infty}}\!\!\!+\alpha(K-1)$.

Finally combining inequality \eqref{eq:proof_deterministic_2} with
 Lemma~\ref{lem:decompose}, we obtain
\begin{align*}
\frac{1}{2}\sum\nolimits_{k=1}^{K}\norm{\tDelta^{(k)}}_F^2 &\leq
2\lambda\sum\nolimits_{k=1}^{K}(\|\mDelta_k'\|_{S_1}+\|\mDelta_k''\|_{S_1})\\
&\leq 8\lambda\sum\nolimits_{k=1}^{K}\|\mDelta_k'\|_{S_1}\\
&\leq 8\lambda\sum\nolimits_{k=1}^{K}\sqrt{2\bar{r}_k}\|\mDelta_k'\|_{F}\\
&\leq 8\lambda\sum\nolimits_{k=1}^{K}\sqrt{2\bar{r}_k}\norm{\tDelta^{(k)}}_{F}\\
&\leq 8\sqrt{2}\lambda\sqrt{\sum\nolimits_{k=1}^{K}\bar{r}_k}\sqrt{\sum\nolimits_{k=1}^{K}\norm{\tDelta^{(k)}}_F^2},
\end{align*}
where we used Lemma~\ref{lem:decompose} in the second line, H\"older's
 inequality in the third line (combined with Lemma~\ref{lem:decompose}),
the fact that $\mDelta^{(k)}_{(k)}=\mDelta_k'+\mDelta_k''$ is an
 orthogonal decomposition in the fourth line, and Cauchy-Schwarz inequality
 in the fifth line.  Dividing both sides of the last inequality by
 $\sqrt{\sum_{k=1}^{K}\norm{\tW^{(k)}}_F^2}$, we obtain our claim.
\end{proof}

\section{Proof of Theorem~\ref{thm:gaussian}}
\label{sec:proof_gaussian}
\begin{proof}
Since each entry of $\tE$ is an independent zero men Gaussian random variable
 with variance $\sigma^2$, for each mode $k$ we have the following 
tail bound (Corollary 5.35 in \cite{Ver10})
\begin{align*}
 P\left(\|\mE_{(k)}\|_{S_\infty}> \sigma\left(\sqrt{N/n_k}+\sqrt{n_k}\right)+t\right)&\leq \exp\left(-t^2/(2\sigma^2)\right).
\end{align*}
Next, taking a union bound
\begin{align*}
& P\left(\max_k\|\mE_{(k)}\|_{S_\infty}>\sigma\max_k\left(\sqrt{N/n_k}+\sqrt{n_k}\right)+t\right)\\
&\qquad\leq K\exp\left(-t^2/(2\sigma^2)\right).
\end{align*}
Substituting $t\leftarrow t+\sigma\sqrt{\log K}$, we have
\begin{align*}
& P\left(\norm{\tE}_{\underline{S_\infty/\infty}}\geq\sigma\max_k\left(\sqrt{N/n_k}+\sqrt{n_k}\right)+\sigma\sqrt{\log
 K}+t\right)\\
&\qquad\leq \exp\left(-\frac{t^2+2\sigma\sqrt{\log K}t}{2\sigma^2}\right)\\
&\qquad\leq\exp\left(-t^2/(2\sigma^2)\right)
\end{align*}
Therefore if $c_0>2$,
\begin{align*}
\lambda&= c_0\sigma\left(\sqrt{N/n_K}+\sqrt{n_1}+\sqrt{\log
 K}\right)+\alpha(K-1)\\
&\geq 2\norm{\tE}_{\underline{S_\infty/\infty}}+\alpha(K-1)
\end{align*}
with probability at least
 $1-\exp\left(-\frac{(c_0-2)^2}{2}(N/n_K)\right)$, which satisfies the
 condition of Theorem~\ref{thm:deterministic}. Substituting the above
 $\lambda$ into the right hand side of the error bound in Theorem~\ref{thm:deterministic} we have the statement of Theorem~\ref{thm:gaussian}.
\end{proof}

\section{Proof of Theorem~\ref{thm:identifiability}}
\label{sec:proof_identifiability}
\begin{proof}
We first prove the ``if'' direction. 
 suppose that there is another decomposition
\begin{align*}
 \sum\nolimits_{k=1}^{K}\tW^{(k)}&=\sum\nolimits_{k=1}^{K}\tilde{\tW}^{(k)},
\end{align*}
such that ${\rm rank}(\mW^{(k)}_{(k)})={\rm
 rank}(\tilde{\mW}^{(k)}_{(k)})$. Note that $\tW\neq\tilde{\tW}$ can
 happen only when $\tW^{(k)}\neq 0$ (otherwise the rank would increase). Also note that $\tW\neq
 \tilde{\tW}$ should happen for at least two $k$'s. Combining these we
 conclude that there are $k\neq \ell$ such that $\tW^{(k)}\neq 0$ and
 $\tW^{(\ell)}\neq 0$.

Conversely, suppose that there are $k\neq
 \ell$ such that $\tW^{(k)}\neq 0$ and $\tW^{(\ell)}\neq 0$, we can
 write\footnote{Here the tensor mode-$k$ product $\tA=\tB\times_k \mC$
 is defined as $a_{i_1\ldots
 i_K}=\sum_{\ell=1}^{d_k}b_{i_1i_2\ldots\ell\ldots i_K}c_{\ell i_k}$
 where $\tA=(a_{i_1\ldots i_K})\in\RR^{n_1\times\cdots\times n_K}$, 
$\tB=(b_{i_1\ldots \ell\ldots i_K})\in\RR^{n_1\times\cdots \times
 d_k\times\cdots\times n_K}$, and $\mC=(c_{\ell i_k})\in\RR^{d_k\times n_k}$}
\begin{align*}
 \tW^{(k)}&=\tC^{(k)}\times_k \mU_k,\\
 \tW^{(\ell)}&=\tC^{(\ell)}\times_{\ell}\mU_{\ell},
\end{align*}
where $\mU_k\in\RR^{n_k\times \bar{r}_k}$, $\tC^{(k)}\in\RR^{n_1\times
 \cdots\times n_{k-1}\times\bar{r}_k\times\cdots \times n_K}$, and
 $\mU_{\ell}$ and $\tC^{(\ell)}$ are defined similarly. Since
 $\tC^{(k)}$ and $\tC^{(\ell)}$ are allowed to be full rank, we can
 define
\begin{align*}
 \tilde{\tC}^{(k)}&=\tC^{(k)} + \tD^{(k,\ell)}\times_{\ell}\mU_{\ell},\\
 \tilde{\tC}^{(\ell)}&=\tC^{(\ell)}-\tD^{(k,\ell)}\times_{k}\mU_{k},
\end{align*}
for any
 $\tD\in\RR^{n_1\times\cdots\times\bar{r}_k\times\cdots\times\bar{r}_{\ell}\times\cdots\times
 n_K}$. Then we have
\begin{align*}
 \tW^{(k)}+\tW^{(\ell)}&=\tC^{(k)}\times_k
 \mU_k+\tC^{(\ell)}\times_{\ell} \mU_{\ell}\\
&=\left(\tC^{(k)} +
 \tD^{(k,\ell)}\times_{\ell}\mU_{\ell}\right)\times_k\mU_k\\
&\qquad +\left(\tC^{(\ell)}-\tD^{(k,\ell)}\times_{k}\mU_{k}\right)\times_{\ell}\mU_{\ell}\\
&=\tilde{\tC}^{(k)}\times_k\mU_k +
 \tilde{\tC}^{(\ell)}\times_{\ell}\mU_{\ell}\\
&=\tilde{\tW}^{(k)}+\tilde{\tW}^{(\ell)}.
\end{align*}
Note that ${\rm rank}(\tilde{\mW}^{(k')}_{(k')})=\bar{r}_{k'}$ for
$k'=k,\ell$. Therefore, there are infinitely many decompositions that
have the same rank $(\bar{r}_1,\ldots,\bar{r}_K)$.


\end{proof}

\end{document}